%% file: chromatic.tex
\documentclass{article}

\usepackage[nonatbib,preprint]{neurips_2020}
\usepackage[square,numbers]{natbib}
\usepackage{wrapfig}
\usepackage[utf8]{inputenc} 
\usepackage[T1]{fontenc}    
\usepackage{hyperref}       
\usepackage{url}            
\usepackage{booktabs}       
\usepackage{amsfonts}       
\usepackage{nicefrac}       
\usepackage{microtype}      
\input{defs.tex}
\usepackage{amsthm}
\newtheorem{theorem}{Theorem}

\usepackage{xcolor}

\title{Chromatic Learning for Sparse Datasets}

\author{%
    Vladimir Feinberg\\
    Sisu Data\\
    \texttt{vlad@sisu.ai}\\
    \And
    Peter Bailis\\
    Sisu Data\\
    \texttt{peter@sisu.ai}
}

\begin{document}

\maketitle

\begin{abstract}
Learning over sparse, high-dimensional data frequently necessitates the use of specialized methods such as the hashing trick. In this work, we design a highly scalable alternative approach that leverages the low degree of feature co-occurrences present in many practical settings. This approach, which we call Chromatic Learning (CL), obtains a low-dimensional dense feature representation by performing graph coloring over the co-occurrence graph of features---an approach previously used as a runtime performance optimization for GBDT training \cite{lgbm}. This color-based dense representation can be combined with additional dense categorical encoding approaches, e.g., submodular feature compression, to further reduce dimensionality \cite{naivesub}. CL exhibits linear parallelizability and consumes memory linear in the size of the co-occurrence graph. By leveraging the structural properties of the co-occurrence graph, CL can compress sparse datasets, such as KDD Cup 2012, that contain over 50M features down to 1024, using an order of magnitude fewer features than frequency-based truncation and the hashing trick while maintaining the same test error for linear models. This compression further enables the use of deep networks in this wide, sparse setting, where CL similarly has favorable performance compared to existing baselines for budgeted input dimension.
\end{abstract}

\section{Introduction}

Extremely sparse, high-dimensional datasets pose significant challenges to resource-efficient learning. In practice, these sparse datasets often arise by combining several disparate data sources, resulting in set-valued features \cite{haldar2019applying, guo2016entity, covington2016deep}. Representing small subsets from a large set with a binary characteristic vector results in many zero-valued entries. For example, \texttt{kdd12}, a popular user behavior prediction dataset, describes topics a user has liked and other users they follow. The dataset has over 50M features in total but each row has at most 10 non-zero features. Many methods for learning over dense inputs, such as neural networks, remain largely intractable in this non-sequential sparse setting: even when restricting to the most frequent 1M features of \texttt{kdd12}, the Wide and Deep architecture \cite{widedeep} requires over 16GB of GPU memory for a 256-element minibatch, meaning these architectures computationally do not scale to the highly sparse regime. Recent work has developed techniques to learn directly from sparse datasets by encouraging model sparsity, such as Neural Factorization Machines \cite{nfm}, but still requires inputs of fewer than 100K features due to memory constraints on input representation given modern hardware. 


In this work, we leverage the structure of many real-world sparse datasets to demonstrate a novel dimensionality reduction technique. We observe that pairs of features rarely co-occur; in
\texttt{kdd12} over 99.9997\% of all possible feature pairs never appear in the any example simultaneously. LightGBM~\cite{lgbm} uses this observation as a means of improving the runtime performance of training gradient-boosted decision trees (GBDTs). We generalize this approach to obtain accuracy improvements in the low-memory regime for a range of models beyond GBDTs, including linear models and deep networks, which depend on dense inputs. In contrast with the popular \emph{hashing trick}~\cite{ht1989,vw} that uses random hashing to reduce dimensions, this data-dependent approach exploits dataset structure to unlock substantial improvements in accuracy on real-world sparse datasets.

Our method, \emph{Chromatic Learning} (CL), performs graph coloring to obtain a low-dimensional, dense, categorical feature representation, then applies additional dense reduction methods. First, CL creates a feature co-occurrence graph, where any features that co-occur have an edge connecting them. Since co-occurrence is rare, the resulting graph is sparse and has a low chromatic number. CL assigns each input feature a categorical variable based on its color in the graph coloring, thus representing sparse inputs using fewer categorical dimensions. This reduces memory usage because identically colored categories share the same embedding dimensions. Subsequently, CL applies one of several techniques for compressing dense categorical features, such as frequency-based truncation and submodular feature compression~\cite{naivesub}.  
By enabling the application of these categorical feature compression techniques in sparse settings, CL exhibits substantial reductions in column count with equivalent accuracy compared to baselines with larger column budgets.  Furthermore, representations from CL generalize because unseen examples exhibit the same sparse structure and rarely contain features identified with the same color.

We demonstrate the efficacy of CL by learning linear models directly on the compressed space of colors and achieve the same level of test error using $10\times$ fewer features than frequency-based truncation and hashing trick approaches across four benchmark datasets. Additionally, we show that in low-dimensional settings CL improves classification performance for a variety of neural network architectures including Factorization Machines \citep{fm}, Wide and Deep learning \citep{widedeep},  Neural Factorization Machines \citep{nfm}, and DeepFM \cite{deepfm}
  compared to baseline dimensionality reduction methods including the hashing trick.


\section{Related Work}\label{sec:rw}

Our contributions relate to several lines of research literature: hashing-based kernels, gradient-boosted trees, and submodular optimization. In this section, we review each.

\subsection{Hashing Trick}

The hashing trick (HT) initially appeared in \cite{ht1989} as a method for dimensionality reduction. HT is a linear transformation $\phi:\R^D\rightarrow\R^d$ which reduces sparse vectors in a high $D$-dimensional space to a small $d$-dimensional one using two hashes $\eta:[D]\rightarrow [d], \xi:[D]\rightarrow \{\pm 1\}$ \citep{ht2009}, $\phi_i(\vx)= \sum_{j:\eta(j)=i}\xi(j)x_j$,
which approximately preserves linear inner products (i.e., $\E\ha{\inner{\phi(\vx), \phi(\vy)}}=\inner{\vx, \vy}$ with low variance) and thus reconstructs a linear kernel on the original space $\R^D$.

Recently, HT structural requirements were characterized by the upper bound on the ratio $\nicefrac{\norm{\vx}_\infty}{\norm{\vx}_2}\le \nu$ for all inputs $\vx$ \citep{ht2018}: with probability $1-\delta$, if $d=\Omega\pa{\epsilon^{-2}\log\frac{1}{\delta}}$ and $\nu=\tilde O\pa{\sqrt{\epsilon}}$, then the relative error between $\norm{\phi(\vx)}_2$ and $\norm{\vx}_2$ is at most $\epsilon$ (a condition related to preserving inner products through the parallelogram law). A $\nu=\tilde O\pa{\sqrt{\epsilon}}$ condition points to the generality of HT. For $k$-sparse binary vectors in $\R^n$, $\nu\le k^{-1/2}$; however, this relies on at least $k$ non-zero values for every sparse input. For a bag of words representation, a $\nicefrac{1}{k}$ relative norm error would require sentences of at least $k$ words. An intuitive question that we seek to answer in this work is whether there are dimensionality reduction mechanisms that can take advantage of stronger structure than $\nu$, and even benefit from having few non-zero entries.

\subsection{Exclusive Feature Bundling}\label{sec:efb}

Gradient-boosted decision trees (GBDTs) are a supervised learning algorithm for learning over a base decision tree (DT) \citep{gbm}. For a weak hypothesis class of DTs $\mcF$, GBDTs provide a mechanism to learn in the larger class of the additive closure $\mcG$ of $\mcF$, where, given a running weighted sum $F_t=\sum_{i=1}^tw_if_i$ for $f_i\in\mcF$, a new $f_{t+1}$ is fit to the gradient of the loss at at each data point  $\partial_{F_t}(\vx)\ell(F_t(\vx), y)$ and then added $F_{t+1}=F_t+w_{t+1}f_{t+1}$ with an appropriately-chosen weight.

While this approach allows search over the larger class $\mcG$, DT training at each iteration is expensive, unless alleviated through specialized techniques like exclusive feature bundling (EFB) in \cite{lgbm}. In that work, the authors notice that DT training can be accelerated by reducing dataset width. EFB builds a co-occurrence graph on the set of categorical features $V$ with an edge between two vertices if they ever co-occur (i.e., attain nonzero values in at least one training example). Coloring this graph $G=(V, E)$ provides a map from vertices to colors, where vertices sharing a color must be mutually exclusive among all observed training examples. \cite{lgbm} uses an incremental but serial binary adjacency matrix construction for an $O(V^2)$ greedy coloring run time.

Unifying each of these color sets into a single categorical variable reduces DT training cost, improving GBDT training speed overall. Similar dimensionality reduction techniques have been observed in other domains, such as register allocation \citep{register}.
However, many estimators require numerical inputs, such as linear models, Gaussian Process classifiers, Factorization Machines, and neural networks. Broader application of EFB is thus limited because one-hot encoding, typically used for pre-processing categorical features, inverts bundling: the one-hot encoding of a bundled feature is equivalent to the concatenation of one-hot encodings for its constituents. 
In this work, we present a method that adapts the general idea of compression via coloring co-occurrence but is suitable for use in models that require numerical inputs.

\subsection{Submodular Optimization for Categorical Feature Compression}\label{sec:rwsubmod}

\begin{wrapfigure}[25]{r}{5cm}
\centering
\vspace{-10pt}
  \caption{Training and test curves with and without sample splitting using categorical feature compression. The Criteo dataset was preprocessed by taking logarithms of its count features, which is standard for this dataset.}\label{fig:criteo-overfit}
\includegraphics[width=4cm]{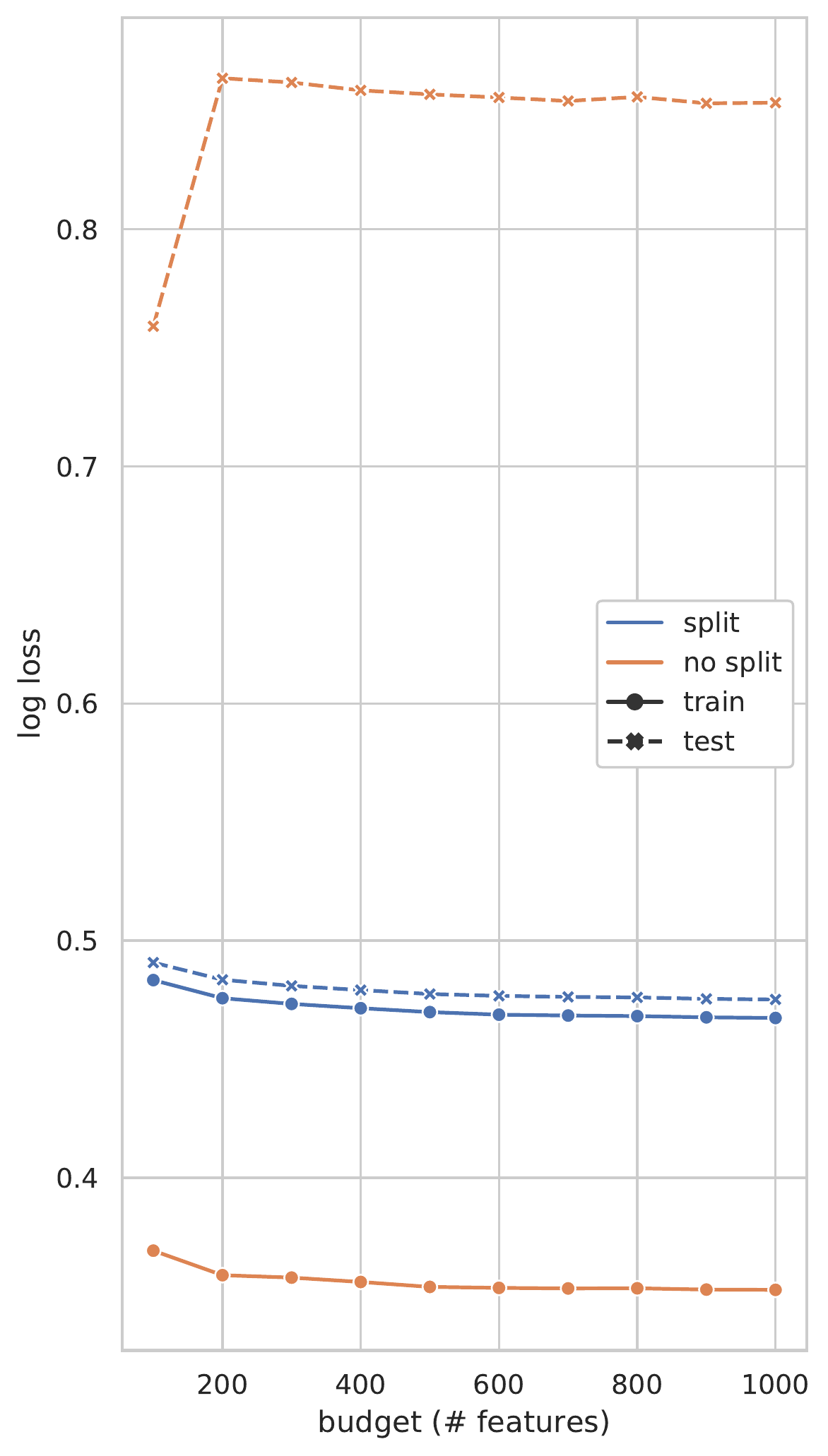}
\end{wrapfigure} 

Recent work \citep{naivesub} shows that the problem of quantizing a fine categorical variable with $D$ values to one with $d$ values while preserving as much mutual information with the binary label as possible is monotone submodular.

A set-valued function $f:2^V\rightarrow\R$ is monotone if $f(A)\le f(B)$ for $A\subset B\subset V$ and submodular if the gain $\Delta(x|Y)=f\pa{\ca{x}\cup Y}-f(Y)$ satisfies $\Delta(x|A)\ge \Delta(x|B)$ for $A\subset B\subset V$ and $x\in V\setminus B$. Such functions admit deterministic $\pa{1-e^{-1}}$-approximate maximization procedures in psuedo-polynomial time \citep{lazygreedy} and polynomial randomized algorithms \cite{stochasticgreedy} for finding $\argmax_{T}f(T)$.

In a classification setting with a single $[D]$-valued categorical feature $X$ and a binary label $Y$, \cite{quantization} show that the mapping $Z:[D]\rightarrow[d]$ which maximizes $I(Z(X);Y)$ is defined by $d+1$ splitters $S=\ca{s_0\cdots s_d}$ with $Z_S(x)=i$ when $s_{i-1}<x\le s_i$. \cite{naivesub} prove that selecting the set $S\subset[D]$ is monotone submodular maximization problem $I(Z_S(X); Y)$.

We found direct application of the method unsuccessful.
Using label information to featurize results in target leakage \citep{catboost}. \cite{naivesub} groups features together with similar conditional positive label probability. This artificially reduces label variance within each quanitzed feature cluster and results in overfitting. We verify this on the original dataset used in \cite{naivesub}, the Criteo Ad-Click Prediction dataset. To reduce variance from the fitting procedure for evaluation, we train a logistic regression, but note that for such a cheap learner, dimensionality reduction is not necessary for the end classification goal. Figure.~\ref{fig:criteo-overfit} shows the training and testing log loss of categorical feature compression applied to estimates of conditional probability made from the training set itself compared to a simple fix, data splitting. With an even split, we estimate conditional probabilities with half of the data, and train on the other half.\footnote{On a practical note, this can be done deterministically without requiring another copy of the data by hashing each example and using the first bit of the resulting hash to split training data.} We find that splitting is crucial for low test loss and, as expected, training loss is conspicuously low when double-dipping.

A few challenges with the application of \cite{naivesub} remain:

\begin{enumerate}
    \item Many datasets (Sec.~\ref{sec:eval}) have millions of binary features, rather than few categorical ones with many values.
    \item \cite{naivesub} extend their method to multiple categorical variables with a heuristic, but it does not perform well when significantly reducing dimension (Sec.~\ref{sec:global-submod}).
\end{enumerate}

\section{Chromatic Learning}\label{sec:cl}

We present CL for reducing the dimensionality of sparse datasets for use in models that expect dense, numerical inputs. For simplicity of presentation, we assume dense, continuous features are set aside and focus on reducing the dimensionality of large, sparse vectors with binary features. 

CL requires performing a parallelizable reduction over the data to collect the set of features $V$ and feature co-occurrences $E$ that comprise the graph $G$ from Sec.~\ref{sec:efb}. After this collection, the graph $G$ is colored. On $P$ processors, this requires $O(V+ E+ P^2)$ memory and $O\pa{\nicefrac{ \Delta V + E + n}{P}}$ serial run time for $n$ examples, where $\Delta$ is the maximum degree in $G$.

At this point, each feature in $V$ is identified with a color, through a mapping $c$, representing a dense categorical dataset over $\card{cV}$ variables and $V$ distinct categorical values across all variables. This enables the application of several categorical encodings. We describe an extension to submodular feature compression and refer to \cite{targetstats} for a description of target encoding. 

\subsection{Chromatic Representation}\label{sec:coloring}

Suppose our training set consists of iid examples, $\vx_i\in 2^D$. Associate with each $\vx_i$ its set of active indices $T_i\subset [D]$, and consider the co-occurrence graph $G$ defined by vertices $V=\bigcup_i T_i$ and edges $E=\bigcup_iK(T_i)$, with $K(\cdot)$ generating the edges from the complete graph on its argument, a set of vertices. Given a proper coloring of $G$, $c:V\rightarrow\N$, we show that the representation of the data defined by categorical vectors $\vv$ with $\card{cV}$ categorical variables (where $v_i$ has cardinality $\card{c^{-1}\{i\}}$) permits learning with low generalization error.

In particular, any Lipschitz-smooth decision function on the original space $2^D$ can be approximated by one that operates on the chromatic representation $\vv$. By construction this holds for all training examples, but for a test example $T\subset[D]$, two features may have identical colors. In this case, the example must be approximated by an input with one of the colliding features missing.\footnote{Given two features mapping to the same color outside the training set, the more popular feature is dropped for our evaluation.}

For a given $T\subset[D]$, let $\mathrm{CC}(G, T)$ be the count of color collisions, i.e., $\card{T}-\card{cT}$. By smoothness, bounding $\E \mathrm{CC}(G, T)$ implies low average discrepancy between the decision function on the two representations (Appendix~\ref{sec:reduced}). We find that a greedy coloring has an average color collision count less than one on all our sparse benchmark datasets (Tab.~\ref{tab:graph-data}).

\begin{table}[h]

  \caption{Graphical properties derived from the training split of benchmark datasets (Table.~\ref{tab:datasets}). $V$ is the set of vertices for the graph generated from these sets containing the unique sparse binary features found in the data. CC refers to the average collision count per example in the test set, using a greedy coloring of the graph generated by the training set.}
  \label{tab:graph-data}
  \begin{center}

  \begin{tabular}{lrrrrrrr}
    \toprule
dataset & avg edges per ex. & $\card{V}$ & avg. degree & colors & CC & avg. nnz  \\
\midrule
\texttt{url} & 481 & 2.71M & 74 & 395 & 0.40 & 29 \\
\texttt{kdda} & 682 & 19.31M & 129 & 103 & 0.30 & 36 \\
\texttt{kddb} & 439 & 28.88M & 130 & 79 & 0.21 & 29 \\
\texttt{kdd12} & 27 & 50.33M & 32 & 22 & 0.09 & 7 \\
    \bottomrule
  \end{tabular}
\end{center}
\end{table}

\subsection{Scalable Coloring}\label{sec:graph-constr}

The previous section motivates constructing the graph $G$ explicitly to obtain a proper coloring. However, construction of $G$ can potentially be expensive. We describe how to efficiently construct $G$.

First, we require the union $E$ of the edge sets $K(T_i)$. We process the dataset in $P$ parallel chunks on mapper threads. We maintain $W=\alpha P$ global sets of edges, where $\alpha$ is the average ratio of time it takes to generate edges from each point to the time it takes to add such edges to a hashset. 
We split the hash space of edges into $W$ parts. Each mapper allocates $W$ local buffers and, upon processing $T_i$ adds edges from $K(T_i)$ to the local buffer corresponding to each edge's shard. Once a buffer is filled, the mapper locks the corresponding global set of edges and adds the global set with its local edges.

The expected contention time on each global set's mutex is constant by the choice of $\alpha$.\footnote{In practice, we can choose $\alpha=1$ as the hashing the edge, which is done locally on each mapper, is the most expensive operation.} As a result, every edge only requires a constant amount of time to process. Each mapper only requires $O(P)$ local working memory, independent of the data, in contrast to tree-merge paradigms such as map-reduce. 
This hash-join approach can be applied to vertices as well and can be extended by storing aggregate statistics for each vertex, as is necessary for Sec.~\ref{sec:global-submod}.

Once $E$ is collected and converted in parallel to an adjacency list representation, $G$ is ready for approximate graph coloring, which requires $O\pa{\nicefrac{V\Delta}{P}})$ time, where $\Delta$ is the maximum degree \cite{parallelcoloring}. Converting to an adjacency list in parallel can be done by mapping over the edge set: first computing the degree of each vertex across threads with atomic increments, and next using a cumulative sum across the degree array to define atomic integral offsets in an adjacency array, which can then be filled in another parallel sweep across $E$.

\subsection{Categorical Encodings}\label{sec:global-submod}

With each feature assigned a color, we view the dataset as a categorical input, with colors as categorical variables and the original input features as categorical values. This permits the use of several categorical encodings.

With target encoding (TE), the average label value replaces categories, yielding just $m$ numeric features, one for each color. With CL combined with frequency truncation (CL+FT) and a budget $b$, we embed the $b$ most frequent categories in $d$ dimensions, yielding an embedding layer that creates $(dm)$-sized embeddings with $bd$ parameters. For neural networks with a first hidden layer of size $h$, using frequency truncation (FT) alone requires the first layer to use $hbd$ parameters, which is prohibitive for all but small $b$.

In addition, we present an extension to submodular-based feature compression (SM). We depart from \cite{naivesub}, which recommends compressing to a fixed budget of features $b$ by running categorical feature compression on each categorical variable $X_i$ to maximize $I(Z_i(X_i); Y)$, sorting results by marginal gain across all features, and using the top $b$ selections for the final encoding. This results in suboptimal compression relative to our alternative, as the final sorting stage is based on marginal gains made with respect to each individual solution, so the top-$b$ features do not necessarily maintain any optimality properties.

Instead, we maximize the global submodular problem $\sum_iI(Z_i(X_i); Y)$ across all colors $i$ simultaneously---adding submodular functions over disjoint inputs retains submodularity. This outperforms the sorting heuristic on the original Criteo task proposed by \cite{naivesub} (Fig.~\ref{fig:criteo-improve}). Excerpting logarithmic factors, solving one submodular problem of size $V$ is faster than multiple of mixed sizes summing to $V$, requiring $ O(\log^2 V + V/P)$ time \citep{balkanski2019exponential}.

With this final encoding, an example $\vx\in\R^D$ is transformed into one in $\R^b$ by taking every nonzero feature $j$ in $\vx$ and setting the corresponding feature $Z_{c(j)}(j)$, which is then one-hot encoded.

\begin{figure}
  \centering
  \includegraphics[width=8cm]{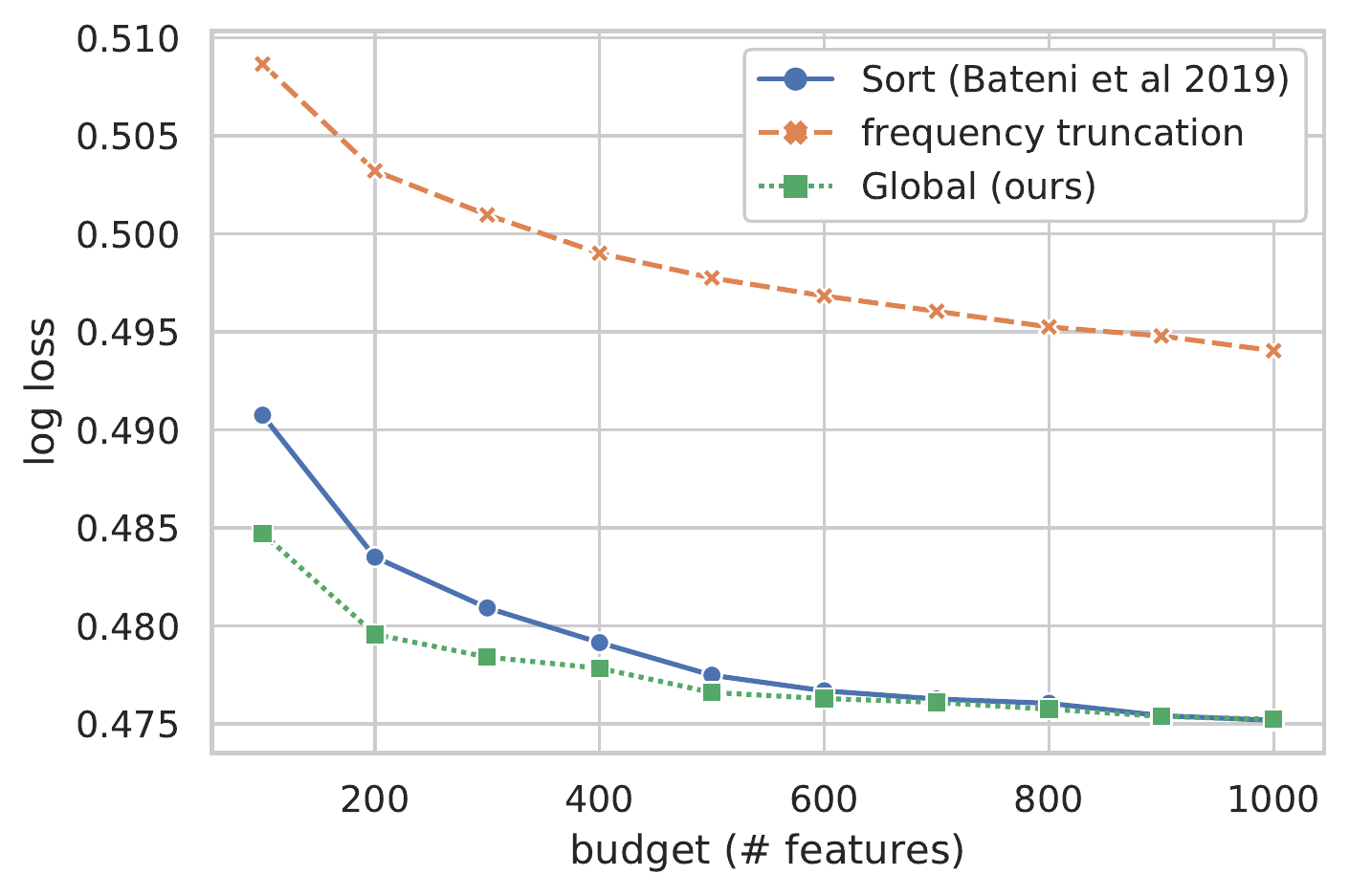}
  \caption{Performance on Criteo ad-click prediction, a dense dataset with 13 numeric columns and 26 high-cardinality categorical columns, evaluated on a linear model at different budgets for encoded vocabulary, comparing our sum-of-mutual information objective to the sort-based one of \cite{naivesub}.}\label{fig:criteo-improve} 
\end{figure}

\section{Evaluation}\label{sec:eval}

In this section, we evaluate CL relative to alternative approaches to learning over sparse data by assessing the accuracy recovered by the same learning procedure applied across several reduction procedures at different budgets for the output dimension.

While our approach allows end-to-end parallelism, we found that on real datasets coloring, initialization, and vertex (feature) processing were not the bottlenecks and did not require parallelism.
We considered a feature dense if it appeared in over $10\%$ of training rows. Sparse features with multiple associated numerical values were treated identically.\footnote{Initial results showed this did not affect performance, so we elided these values. An alternative would track distinct feature values with a separate dictionary, binning for sparse and continuous data.}

We evaluate standard sparse benchmark datasets available in Table~\ref{tab:datasets}. All four evaluation sets contain millions of sparse binary features, with relatively few active per example. Only the \texttt{url} dataset \citep{url} had numeric features. No additional preprocessing was performed on the retrieved data.\footnote{Datasets may be retrieved from \texttt{https://www.csie.ntu.edu.tw/\textasciitilde cjlin/libsvmtools/datasets/}.}

\begin{table}
  \caption{Dataset dimensions. All datasets were generated from sequential observations, so they were split chronologically. Average nnz denotes the average number of nonzero entries per row.}
  \label{tab:datasets}
  \centering
  \begin{tabular}{lrrrrrrr}
    \toprule
dataset & train ex. & test ex. & dense feat. & sparse feat. & avg. nnz & max nnz \\
\midrule
\texttt{url} & 1.68M & 0.72M & 134 & 2.71M & 29 & 327 \\
\texttt{kdda} & 8.41M & 0.51M & 0 & 19.31M & 36 & 85 \\
\texttt{kddb} & 19.26M & 0.75M & 2 & 28.88M & 29 & 75 \\
\texttt{kdd12} & 119.71M & 29.93M & 7 & 50.33M & 7 & 10 \\
    \bottomrule
  \end{tabular}
\end{table}

\subsection{Dimensionality Reduction}\label{sec:dimred}

A comparison of chromatic learning with submodular feature compression (CL+SM) to frequency-based truncation (FT) and hashing trick (HT) shows favorable performance in for linear estimators across a variety of budgets. We performed a single pass over the training data via Vowpal Wabbit with default parameters using, which performs online adaptive gradient descent \cite{vw}. Our results illustrate that, to achieve equivalent log loss on the test set, FT requires a magnitude more features and HT requires two orders of magnitude (Fig.~\ref{fig:linear}).

\begin{figure}
  \includegraphics[width=\columnwidth]{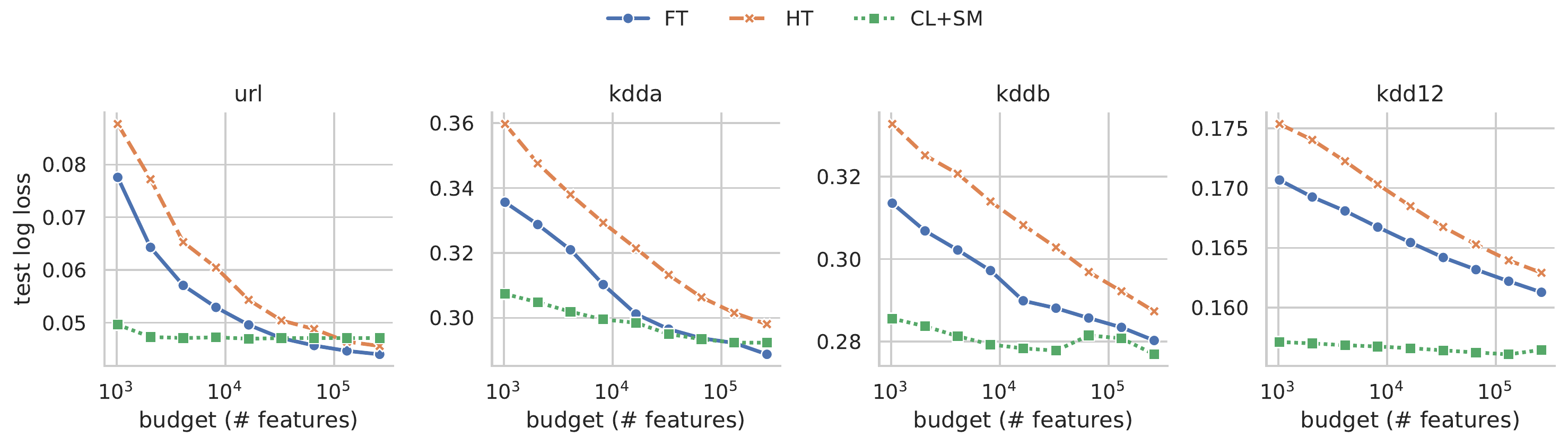}
  \caption{Chromatic learning outperforms FT and HT for all input feature budgets under $2^{18}$ for datasets with significantly more training examples (\texttt{kddb} and \texttt{kdd12}). For datasets with single-digit millions of training points (\texttt{url} and \texttt{kdda}), CL performs better than the baselines on small budgets of input features but slightly worse for large ones. Since CL is using half the training points for supervised learning than HT or FT are, it hits a generalization error floor earlier, but this floor can be lowered by tuning the split ratio mentioned in Sec.~\ref{sec:rwsubmod}.
}\label{fig:linear}
\end{figure}

For very large budgets that approach the original dimensionality on the two smaller datasets, FT outperforms CL+SM. In this regime, where nearly the full input dimensionality is preserved, we suspect that the quantization may interfere with model confidence.

\subsection{Non-Linear Sparse Learning}\label{sec:non-linear}

With a small budget, CL enables learning more sophisticated classifiers that are otherwise not trainable in limited-memory settings. We assess the performance of several deep learning models given different dimensionality reduction approaches in the highly compressed setting of 1024 features. 

We used default parameters specified by the deep learning library \cite{pytorchfm}, but reduced batch size to 256 (from 2048) because of an out-of-memory error on the Nvidia V100 GPUs we were using. The library was originally configured for dense datasets, such as Criteo.
We reduced the epoch count to 5 (from 15) to account for the increased number of gradient steps and did not change any other settings.

We find that CL compares favorably in terms of test accuracy with several categorical encoding strategies compared to the FT baseline (Tab.~\ref{tab:nonlinear}). Across all 5 evaluated architectures and all 4 datasets, CL+SM outperforms both HT and FT on every dataset, except for DeepFM with \texttt{url}, where it achieves a test log loss of 0.058, compared to FT's 0.053. Furthermore, CL+FT presents a simpler alternative to CL+SM which enables using the full dataset and many more features directly, on 2 of the 4 datasets, this improves upon the CL+SM representation with 1024 one-hot features.



\begin{table}
  \caption{Final classifier test log loss across different models. Each model compares frequency-based truncation (FT) and hashing trick (HT) baselines against chromatic learning (CL) approaches of submodular feature compression (SM) and target encoding (TE), with input dimension budget set to 1024. We notice that a CL with frequency-based truncation with large budgets $2^{18},2^{20}$ (CL+FT@18, CL+FT@20) performs better than the more sophisticated CL+SM approach on certain datasets. For such large budgets, the Wide and Deep model with just FT runs out of memory or does not finish training within a day on a V100 Nvidia Tesla GPU.}
  \label{tab:nonlinear}
  \centering
\begin{tabular}{llrrrr}
\toprule
learner & encoder & \texttt{url} & \texttt{kdda} & \texttt{kddb} & \texttt{kdd12} \\
\midrule
Wide and Deep \citep{widedeep}  & FT &  0.046 & 0.326 & 0.314 & 0.172 \\
 & HT &  0.058 & 0.337 & 0.328 & 0.173 \\
 & CL+SM &  0.037 & 0.296 & \textbf{0.276} & \textbf{0.159} \\
 & CL+TE &  0.325 & 0.588 & 0.517 & 0.280 \\
 & CL+FT@18 &  0.032 & \textbf{0.292} & 0.296 & 0.169 \\
 & CL+FT@20 &  \textbf{0.030} & \textbf{0.292} & 0.291 & 0.169 \\
\midrule
Logistic Regression & FT &  0.077 & 0.336 & 0.315 & 0.172 \\
 & HT &  0.087 & 0.359 & 0.330 & 0.175 \\
 & CL+SM &  \textbf{0.048} & \textbf{0.308} & \textbf{0.283} & \textbf{0.159} \\
 & CL+TE &  0.198 & 0.571 & 0.423 & 0.204 \\
\midrule
Factorization Machines \citep{fm} & FT &  0.076 & 0.321 & 0.311 & 0.172 \\
 & HT &  0.086 & 0.339 & 0.323 & 0.170 \\
 & CL+SM &  \textbf{0.063} & \textbf{0.303} & \textbf{0.282} & \textbf{0.159} \\
 & CL+TE &  0.253 & 0.691 & 0.665 & 0.206 \\
\midrule
Neural Factorization Machines \citep{nfm} & FT &  0.045 & 0.320 & 0.309 & 0.172 \\
 & HT &  0.056 & 0.317 & 0.310 & 0.170 \\
 & CL+SM &  \textbf{0.040} & \textbf{0.291} & \textbf{0.271} & \textbf{0.158} \\
 & CL+TE &  0.285 & 0.804 & 0.839 & 0.294 \\
\midrule
DeepFM \cite{deepfm} & FT &  \textbf{0.053} & 0.321 & 0.313 & 0.172 \\
 & HT &  0.075 & 0.333 & 0.327 & 0.170 \\
 & CL+SM &  0.058 & \textbf{0.302} & \textbf{0.287} & \textbf{0.159} \\
 & CL+TE &  0.363 & 0.792 & 0.681 & 0.282 \\
\bottomrule
\end{tabular}
\end{table}

\section{Discussion}\label{sec:disc}

In this work, we explore using graph coloring to generate virtual categorical variables that create a dense view of a sparse dataset. The strong empirical performance (Tab.~\ref{tab:nonlinear}) requires explanation, since the variables were created artificially.

The fact that the colors generated by a greedy coloring are effective at representing the dataset on unseen examples is surprising, because new co-occurrences (edges) appear in the test set frequently (Tab.~\ref{tab:graph-data}). When new edges appear between features of different colors, the chromatic representation is lossless. However, in an adversarial setting, this data-dependent property may not hold. An alternative approach (which our early experiments deemed unnecessary) would choose a random, uniform coloring over $G$ instead. If $G$ has a low maximum degree $\Delta$, then such a coloring can be sampled by simulating Glauber dynamics with $m>2\Delta$ colors \cite{jerrum}. Furthermore, such sampling is internally parallelizable because each update is local to a vertex's neighborhood. Given any new co-occurrence in the test set between two features $x,y$, a uniform coloring must have assigned $x,y$ one of $m-\deg_Gx-\deg_Gy\ge m-2\Delta$ colors, limiting a color collision's probability by choice of $m$ regardless of the properties of the data's distribution. However, greedy coloring required fewer colors than a random one and did not encounter many color collisions in practice. 

Second, how can we reconcile the relationship between the discrete graphical structure $G$ and a variable training set size $n$, which affects $G$? Given infinte datasets, would $G$ be complete? $G$ is a random graph obtained from sampling cliques from a true latent graph $L$ of co-occurrences. Since $G\subset L$, it suffices for the latent graph $L$ to be sparse. However, even if $L$ is complete, edges only contribute to collisions to the extent that they appear in test examples. Thus, for the co-occurence graph $G=(V, E)$, the expected count of unseen edges, $\E\card{K(T)\setminus E}$, provides a better proxy for the propriety of the chromatic representation than $L\setminus G$. Given a set of $n$ training examples, $\E\card{E\setminus K(T)}$ admits a Good-Turing-type \cite{goodturing} upper bound, $\frac{1}{n}\sum_i\card{K(T_i)\setminus E_i}$ where $E_i=\bigcup_{i'\neq i}K(T_i)$. This approach lets us analyze even sparser co-occurrence graphs, such as $G^{(2)}$, the co-occurrence graph with edges that appear at least twice in the training set. $G^{(2)}$ and its natural successors $G^{(3)},G^{(4)},\cdots$ may be efficiently constructed by using a rolling set of bloom filters to prune out edges that are only generated by a few training examples \cite{mcmahan2013ad}. Appendix~\ref{sec:reduced} elaborates on the trade-off between chromatic representation fidelity and colors used by $G^{(1)}=G,G^{(k)}$, and how Good-Turing-type estimators can be used to forecast new edge incidence counts between the two co-occurrence construction approaches.

Besides making learning tractable for estimators that require dense inputs, one interesting implication of CL+SM is that it yields $n\times d$ design matrices where, typically, $d\ll n$. For small $d$, this opens up sketching approaches to learning whose superlinearity in input dimension otherwise makes them inaccessible to wide, sparse datasets, such as coreset construction for nearest-neighbor queries \cite{inaba1994applications}. Further, in the case of linear models, an $n\times d$ design can be represented faithfully as a $(d+1)\times d$ one by Carath\'{e}odory's Theorem \citep{lms}, which for small $d$ can greatly simplify linear learning (e.g., tuning regularization parameters no longer requires multiple passes over the data).

Beyond the supervised setting, co-occurrence graphs may be appealing from an unsupervised learning perspective: a weighted co-occurrence graph may be used to accelerate graphical model structure learning \cite{chordalysis} by pruning the search space of log-linear models during forward selection.


\section{Conclusion}\label{sec:conclude}

We have introduced Chromatic Learning, a method that provides a viable representation of sparse data that enables otherwise-inaccessible learning methods to be applied---such as neural networks---in memory-constrained settings, as shown in Sec.~\ref{sec:non-linear}. This presents several avenues for future work. Optimizing the tradeoff in the data split for submodular feature compression between estimating conditional probabilities and training may result in lower test error. In addition, a balanced coloring scheme may further reduce color collisions, improving accuracy.
Finally, the approaches presented in this work illustrate the co-occurrence graph is recoverable in explicit form for many high-dimensional, sparse datasets. This phenomenon may merit its own study.

\pagebreak
\section*{Broader Impact}

This work provides learning methods that scale effectively across many processors while limiting memory. Such methods encourage more organizations to adopt machine learning techniques because of the relative cost of horizontal versus vertical scaling and the overall cost of memory.

Open access to distributed and large-scale methods is important for leveling the playing field between organizations in general that wish to apply learning techniques on such large, realistic sparse data.

This may be positive or negative, depending one's alignment with the values and goals of each individual organization that may apply chromatic sparse learning. Behind the veil, we believe increased accessibility of large-scale learning through cheaper processing of the same data is net positive.

\begin{ack}
The authors would like to thank Jon Gjengset, Greg Valiant, Barak Oshri, Arnaud Autef, and Kai Sheng Tai for helpful comments.
\end{ack}

\bibliography{chromatic}

\pagebreak
\appendix

\section{Reduced Co-occurrence Graphs}\label{sec:reduced}

In this section, we elaborate on guarantees for using reduced co-occurrence graphs. Recall that we observe $n$ iid samples of vectors $\vx_i\in 2^D$, equivalently represented by sets $T_i\subset[D]$ of active indices, with $K(T_i)$, the edge set of the complete graph on vertices $T_i$. We assume a sparse setting, i.e., $\card{T_i}\le \eta$.

Consider the set of vertices $V=[D]$ and the different co-occurrence graphs constructed over $V$ by using edges that appear at least $k$ times, $G^{(k)}$. These graphs will satisfy the nesting property $G^{(k+1)}\subset G^{(k)}$. Denote the edge count $\#(e)=\card{\set{i\in[n]}{e\in K(T_i)}}$ as well as the leave-one-out edge count $\#_i(e)=\card{\set{i'\in[n]\setminus\{i\}}{e\in K(T_{i'})}}$. This lets us build reduced co-occurrence graphs $G^{(k)}=(V, E^{(k)})$ containing edges $E^{(k)}=\set{e\in \bigcup_i K(T_i)}{\#(e)\ge k}$ and their leave-one-out analogues $G^{(k)}_i=(V, E^{(k)}_i)$ and $E^{(k)}_i=\set{e\in \bigcup_{i'\neq i} K(T_{i'})}{\#_{i}(e)\ge k}$.

Then we may construct Good-Turing-type estimators \cite{mcallester2000convergence} for the expected count of new edges. With $T$ an independent sample from the same distribution as each $T_i$, our estimator $N^{(k)}=\sum_i\card{K(T_i)\setminus E_i^{(k)}}$ upper bounds the average new edge count in expectation $\E N^{(k)}\ge n\E\card{K(T)\setminus E^{(k)}}$ because $E_i^{(k)}\subset E^{(k)}$. When we color a graph $G^{(k)}$ with $m$ colors by a function $c:V\rightarrow[m]$, we induce a chromatic representation in the space $C=c^{-1}\{1\}\times c^{-1}\{2\}\times \cdots\times c^{-1}\{m\}$ (to represent absences, suppose $\bot\in V$ without loss of generality, where $c(\bot)=0$). For simplicity, we consider a fixed lossy transformation between $2^D$ and $C$, given by $\vx\mapsto \vy$, where $y_j=0\lor \inf \set{i}{x_i=1,c(i)=j}$. For certain machine learning algorithms, such as neural networks, learning on $C$ is much cheaper than on $2^D$ directly. Furthermore, categorical dimensionality reduction techniques may be applied to $C$, whereas they would be unavailable in $2^D$.

To analyze the fidelity of the space $C$ in representing vectors from the sampling distribution, we'll consider how lossy their one-hot encoding is in the original $2^D$ space. This transformation yields a color collision resolution function $\chi:2^D\rightarrow 2^D$ that describes achievable values by the chromatic representation. If $\vx$ exhibits no color collisions under $c$, then the collision resolution function $\chi(\vx)=\vx$. If $T$ is the set of active indices in $\vx$ then $\chi$ relates to the color collision count through $\norm{\chi(\vx)}_1=\card{c T}$. Given this setup, we can construct a range of chromatic representations, parameterized by the threshold $k$, with the following representation fidelity property for all $k$.

\begin{theorem}\label{thm:guarantee}
Consider a measure $\mu$ over $2^D$, a sample of $n$ independent observations $\mathbf{x}_i\sim \mu$, whose corresponding sets of active indices $T_i\subset[D]$ satisfy $\card{T_i}\le \eta$, and fix any $k,m\in\N$. As above, construct the random graph $G^{(k)}$ over vertices $V=[D]$ with the union of edges from complete graphs $K(T_i)$ which appear at least $k$ times among all $i\in[n]$. Independently sample a uniform proper $m$-coloring $c$ over $G^{(k)}$. This random coloring induces the collision resolution function $\chi:2^D\rightarrow 2^D$ that, for any input $\vx$, returns the same vector but with all higher indices $x_i$ set to zero if $c(i)=c(j)$ for some smaller index $j<i$ with $x_j=1$. Define $\Delta^{(k)}$ to be the maximum degree of $G^{(k)}$. Then, if $m>2\Delta^{(k)}$, with probability $1-\delta$ over the choice of the sample, for any $L$-Hamming-Lipschitz  $f:2^D\rightarrow \R$,
$$
\E \norm{f-f\circ \chi}_{L^1(\mu)}\le \frac{L}{m-2\Delta^{(k)}}\pa{\frac{N^{(k)}}{n}+\frac{k \eta^2 \log \delta^{-1}}{\sqrt{n}}}\,\,,
$$
where $N^{(k)}$ is defined recursively by $N^{(k)}=k f(k)+N^{(k-1)}$ with $N^{(0)}=0$ and $f(k)$ is the count of edges which appear $k$ times in the multiset $\ca{K(T_i)}_i$.
\end{theorem}
\begin{proof}
We first reduce our error to the collision count. Consider a new, independent sample $\vx\sim \mu$ and its associated set $T\subset[D]$ of active indices, yielding
\begin{align*}
    \E\norm{f-f\circ \chi}_{L^2(\mu)}&=\E\abs{f(\vx)-f\circ\chi(\vx)} & \vx\sim\mu \\
    &\le L\E\norm{\vx-\chi(\vx)}_1 & \text{$f$ Hamming-Lipschitz}\\
    &\le L\E \mathrm{CC}\pa{G^{(k)}, T} & \text{$\chi$ definition}\,\,,
\end{align*}
where $\mathrm{CC}(G^{(k)}, T)$ is the color collision count with respect to the coloring $c$ of $G^{(k)}$, equal to $\card{T}-\card{cT}$. Next, every edge in $K(T)\setminus E^{(k)}$ contributes to at most 1 color collision, so
$$
\E{\mathrm{CC}\pa{G^{(k)}, T}}\le\E\ha{\sum_{xy\in K(T)\setminus E^{(k)}}\CE{1\ca{c(x)= c(y)}}{T_{1:n},T}}\,\,,
$$
where we also apply the tower property. Edges in $K(T)\cap E^{(k)}$ cannot result in a color collision because $c$ is proper. Further, since $c$ sampled independently of $T_{1:n},T$ among all uniform $m$-colorings, for any two vertices $x,y$, which each have at most $\Delta^{(k)}$ neighbors in $G^{(k)}$, there are at least $m-2\Delta^{(k)}$ colors which $x,y$ may be assigned, so $\CE{1\ca{c(x)= c(y)}}{T_{1:n},T}\le \pa{m-2\Delta^{(k)}}^{-1}$. Applying this bound above yields
$$
\E{\mathrm{CC}\pa{G^{(k)}, T}}\le \pa{m-2\Delta^{(k)}}^{-1}\E{\card{K(T)\setminus E^{(k)}}}\,\,.
$$
We turn our attention to the last term $\E{\card{K(T)\setminus E^{(k)}}}$, which is the expected count of new edges. $\card{K(T)\setminus E^{(k)}}\le \card{K(T)\setminus E^{(k)}_i}$ per $E^{(k)}_i\subset E^{(k)}$ and $\card{K(T)\setminus E^{(k)}_i}\disteq \card{K(T_i)\setminus E^{(k)}_i}$ since $T\disteq T_i$ but both are independent of $E^{(k)}_i$. Summing $n$ such terms over $i\in[n]$, we have
$$
\E{\card{K(T)\setminus E^{(k)}}}\le\frac{1}{n}\E N^{(k)}=\frac{1}{n}\E \sum_i\card{K(T_i)\setminus E_i^{(k)}}\,\,,
$$
At this point, combining all of our inequalities, we have shown
$$
\E \norm{f-f\circ \chi}_{L^1(\mu)}\le \frac{L}{m-2\Delta^{(k)}}\E\frac{N^{(k)}}{n}\,\,.
$$
To finish, we show that $\frac{1}{n}N^{(k)}$ concentrates about its mean by applying McDiarmid's inequality, recognizing $N^{(k)}$ is a function of $n$ arguments $T_i\in2^{D}$ satisfying $\card{T_i}\le \eta$. We focus on showing bounded differences by replacing $T_i$ with some $T_i'\subset[D]$, also with $\card{T_i'}\le \eta$, yielding new $\pa{E^{(k)}_j}'=\pa{E^{(k)}_j\setminus K(T_i)}\cup K(T_i')$ for $j\neq i$. The $i$-th term of the sum $N^{(k)}$ changes by $\card{K(T_i')\setminus \pa{E^{(k)}_i}'}-\card{K(T_i)\setminus E^{(k)}_i}$, which is bounded by $\eta^2$, as $K(T_i'),K(T_i)$ are bounded by that size. Next, the remaining $j$-th terms of $N^{(k)}$ for $j\neq i$, $\card{K(T_j)\setminus \pa{E^{(k)}_j}'}$, can only increase by at most one relative to $\card{K(T_j)\setminus E^{(k)}_j}$ for every edge $e\in K(T_i)$. If $\#(e)>k$, then $e\in \pa{E^{(k)}_j}'$ regardless. Otherwise, an edge $e$ in $K(T_i)$ is only absent from $E^{(k)}_j$ if it appears in at most $k-1$ other sets $K(T_j)$. Thus, replacing $T_i$ with $T_i'$ increases $N^{(k)}$ by at most $k\eta^2$. The alternative formulation for $N^{(k)}$ follows by recognizing that it can be computed explicitly as the sum of the count of edges which appear exactly $j\in[k]$ times.\end{proof}

\begin{figure}[h]
  \includegraphics[width=\columnwidth]{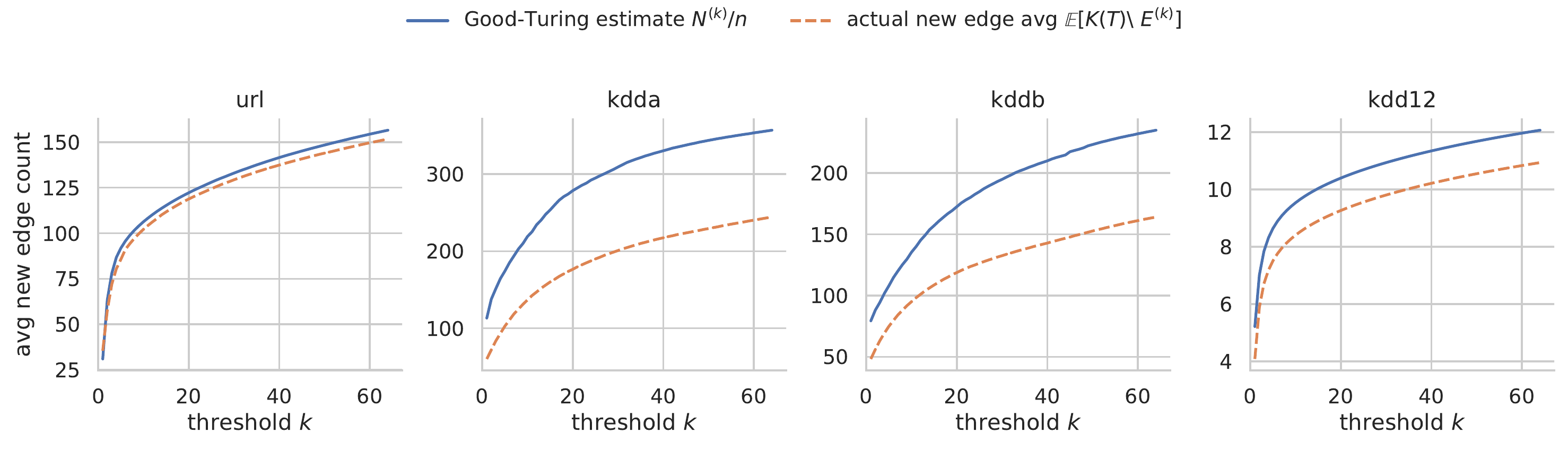}
  \caption{As the sparsity of the reduced graphs $G^{(k)}$ increases for increasing $k$ across the datasets, we notice an increase in the number of new edges introduced by unseen points. Luckily, the estimated new edge count $\frac{1}{n}N^{(k)}$ provides a tractable estimator for the upper bound of the average number of new edges an example introduces, $\E\ha{K(T)\setminus E^{(k)}}$. The empirical estimate of $\E\ha{K(T)\setminus E^{(k)}}$ is based on held-out test data.
}\label{fig:gtworks}
\end{figure}

\begin{figure}
  \includegraphics[width=\columnwidth]{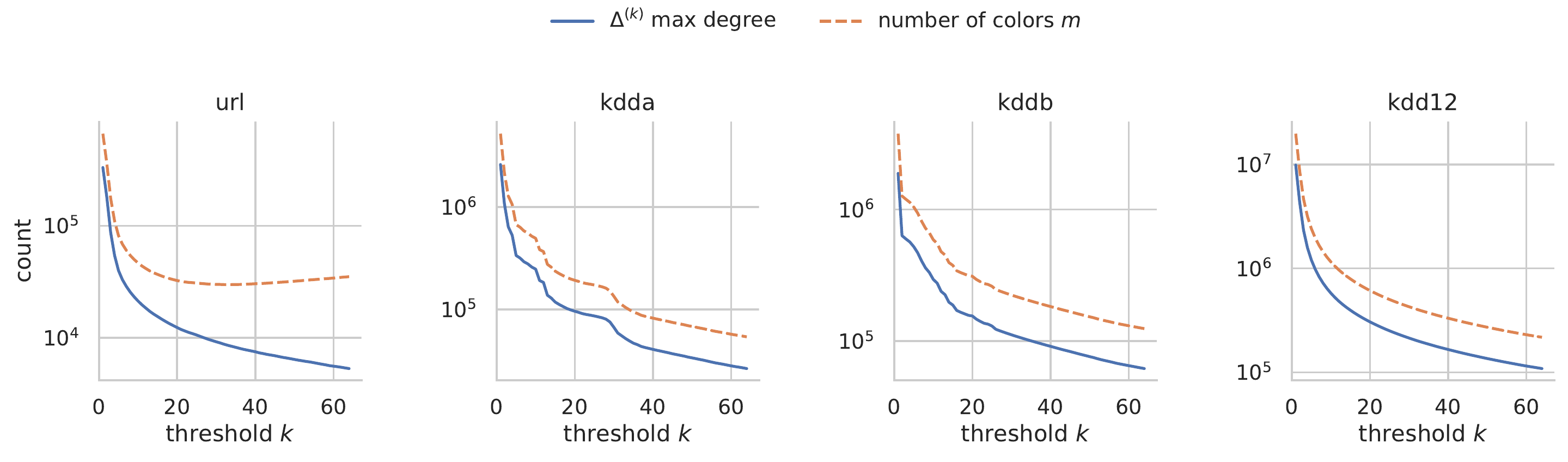}
  \caption{Across all datasets and different thresholds $k$, a large maximum degree $\Delta^{(k)}$ on the reduced co-occurrence graphs pushes the number of requisite colors up as well.
  To have error $L$ for $L$-Lipschitz functions with $99\%$ probability, one requires $m=2\Delta^{(k)}+\frac{N^{(k)}}{n}+\frac{k\eta^2\log 100}{\sqrt{n}}$ colors, which does not reduce dimensionality.
}\label{fig:maxdeghigh}
\end{figure}

Given Thm.~\ref{thm:guarantee}, a smooth kernel machine $k(\cdot, \alpha)$ over $2^D$ can be approximated by one over the chromatic representation $C$. The flexibility afforded by dealing with sparser graphs $G^{(k)}$ allows one to progressively increase $k$ as $n$ increases: the Good-Turing term $N^{(k)}$ provides a data-dependent estimate of unseen edge mass per new example (Fig.~\ref{fig:gtworks}).

Based on Thm.~\ref{thm:guarantee}, the number of requisite colors is at least $2\Delta^{(k)}$. Unfortunately, in practice $\Delta^{(k)}$ tends to be large due to a few vertices with high degree (Fig.~\ref{fig:maxdeghigh}).

By filtering out a small set of vertices $W\subset V$ and applying CL to the induced graph $G_f$, which is $G$ with vertices $W$ removed, we can reduce the maximum degree of $G_f$ significantly, in turn dropping the color lower bound. Vertices in the hold-out set may be colored with their own unique colors. As a result, the lower bound on the number of colors necessary for fidelity $L$ with probability at least 99\% based on Thm.~\ref{thm:guarantee} is $m_f=\card{W}+2\Delta_f^{(k)}+\frac{N_f^{(k)}}{n}+\frac{k\eta^2\log 100}{\sqrt{n}}$, where $G_f^{(k)}$ is the induced graph resulting from removing $W$ from $G^{(k)}$, $\Delta_f^{(k)}$ is its maximum degree, and $\frac{1}{n}N_f^{(k)}$ is the Good-Turing estimator for the average new edge count for unseen test examples on the filtered graph $G_f^{(k)}$. The Good-Turing estimates remain valid, even with filtering (Fig.~\ref{fig:gtfilter}).

To choose the set $W$ which maximally reduces the degree, we require a largest-first order $\sigma$ of the vertices, where $\sigma(i)$ is a vertex of maximum degree in the induced graph on vertices $\sigma(i)\cdots\sigma\pa{\card{V}}$. Such an ordering may be computed in time linear in the graph size \cite{matula1983smallest}.

We remove vertices in largest-first order until the number of removed vertices $\card{W}$ is at least twice the maximum degree of the remaining graph $2\Delta_f^{(k)}$. This results in a significantly smaller number of requisite colors $m_f$ across all thresholds $k$ compared to $m$, even when including the $\card{W}$ distinct colors used for representing high-degree vertices (Fig.~\ref{fig:reducedcolors}).

\begin{figure}
  \includegraphics[width=\columnwidth]{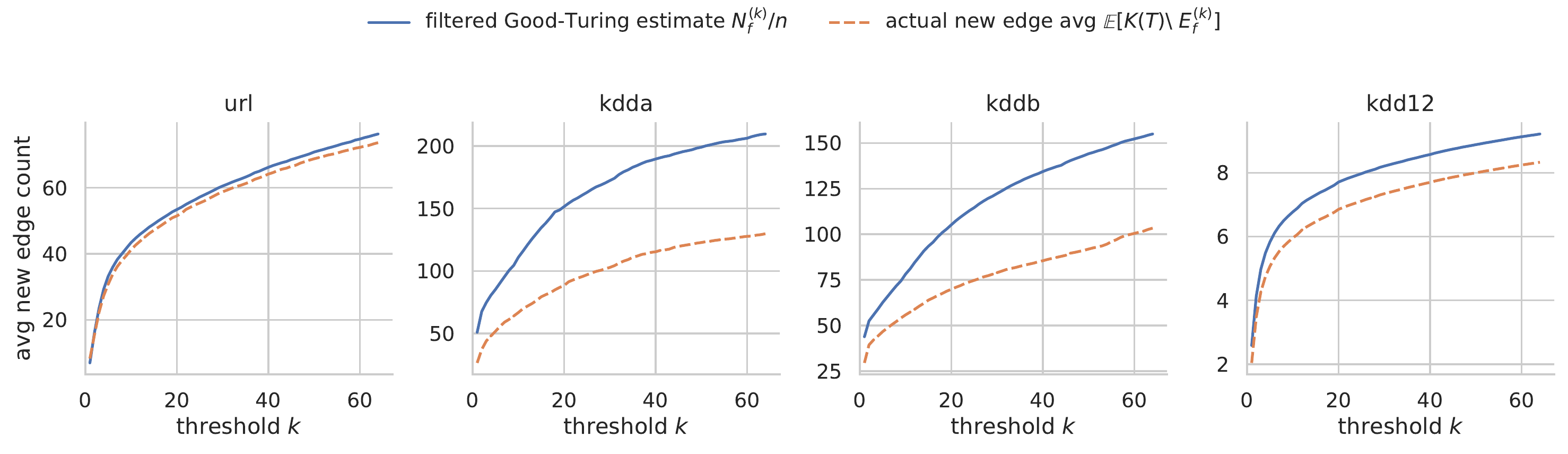}
  \caption{The filtered graphs $G^{(k)}_f$ exclude high-degree vertices from $G^{(k)}$ increases for increasing $k$ across the datasets and as a result have much lower average counts of new edges appearing in held-out data. The Good-Turing estimator $\frac{1}{n}N_f^{(k)}$ applied to the subgraphs remains a valid high-probability upper bound.
}\label{fig:gtfilter}
\end{figure}

\begin{figure}
  \includegraphics[width=\columnwidth]{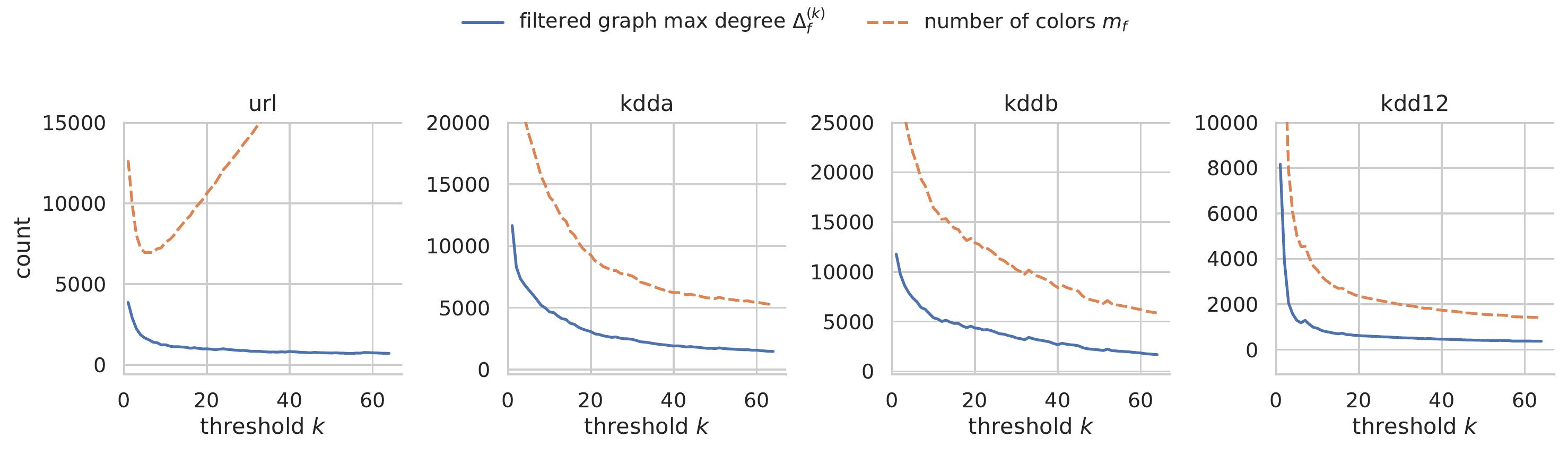}
  \caption{For each thresholded graph $G^{(k)}$, across different thresholds $k$, we use distinct colors for held-out vertices $W$ and use the chromatic representation for the filtered graph $G^{(k)}_f$, which is the induced subgraph of $G^{(k)}$ that excludes $W$. Because the graphs for the above datasets have few high-degree vertices, the filtering strategy significantly reduces the number of requisite colors $m_f$ (defined in the text) for achieving $L$ function approximation error with probability 99\% compared to using no filtering. Note that \texttt{url} quickly starts increasing with $k$ because of its relatively high max nnz term, $\eta^2$.
}\label{fig:reducedcolors}
\end{figure}

End-to-end, we can achieve low color collisions, and thus a high-fidelity representation of the original dataset, independent of the distribution of the data, through the chromatic representation by:
\begin{enumerate}
    \item Choosing an appropriate filtering threshold $k$, given the training set size $n$.
    \item Collecting the co-occurrence graph $G$ (Sec.~\ref{sec:graph-constr}). Using a rolling set of bloom filters allows one to collect the smaller $G^{(k)}$ directly \cite{mcmahan2013ad}, by doing a preliminary pass with $k$ bloom filters tracking whether each edge appeared at least $k$ times, rolling an edge to the next filter if it's present in the prior one, and then doing another pass over the data, filtering by the last bloom filter.
    \item Computing the largest-first ordering of $G^{(k)}$ \cite{matula1983smallest}.
    \item Identifying the smallest prefix $W$ of the largest-first ordering which has size at least twice the maximum degree $\Delta_f^{(k)}$ of the subgraph $G^{(k)}_f$ induced by the rest of the ordering.
    \item Uniformly sampling a coloring of $G^{(k)}_f$ from at least $m\ge 2\Delta_f+\frac{N_f^{(k)}}{n}$ colors by simulating Glauber dynamics for $\tilde O\pa{m\pa{\card{V}-\card{W}}}$ steps \cite{jerrum} and coloring each vertex in $W$ with new colors.
    \item Using the color map to perform categorical encoding, through embeddings or submodular feature compression.
\end{enumerate}

The transition dynamics for the Glauber Markov chain are determined by each vertex's neighborhood. The Markov blanket for each vertex is exactly its neighbors, so MCMC simulation can sample multiple vertices at once. The filtered graph has $v=\card{V}-\card{W}$ vertices with average degree $d=2\card{E_f^{(k)}}/v$. By a birthday paradox calculation, with $P$ such simultaneous MCMC updates, we can expect contention on less than a single vertex at any given time, so long as $P=O\pa{\nicefrac{\sqrt{v}}{d}}$.
Although the end-to-end uniform coloring process above uses more colors and computation than the greedy approach in the main text, it has far fewer color collisions (Fig.~\ref{fig:unifcc}).

\begin{figure}
  \includegraphics[width=\columnwidth]{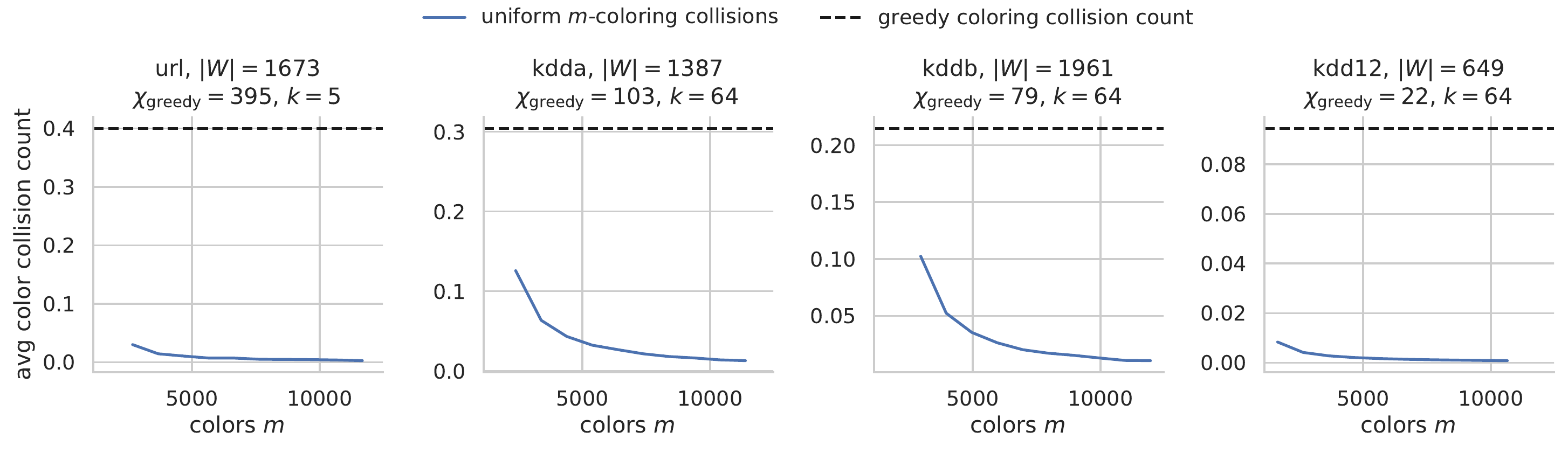}
  \caption{The above demonstrates the estimated average color collision count against the number of colors used when sampling a uniform coloring. As expected per Thm.~\ref{thm:guarantee}, an increase in color count results in a hyperbolic decrease in average color collision count, as estimated by the test set for each dataset. The size of the prefix $W$ of the largest-first ordering of each reduced co-occurrence graph as chosen by the rule $\card{W}\ge 2\Delta_f^{(k)}$ described aobve is specified for each dataset. The number of colors used by the greedy algorithm is given by $\chi_{\mathrm{greedy}}$. The thresholds $k$ for reducing the co-occurrence graphs were chosen based on Fig.~\ref{fig:reducedcolors}.
}\label{fig:unifcc}
\end{figure}

More efficient color sampling on random graphs is an active area of research, suggesting that slightly modified Glauber dynamics can mix quickly with fewer colors, even in the chance presence of high-degree vertices \cite{dyer2006randomly}. Such work leads to opportunities that obviate the need for the discrete representation of a prefix of largest-first vertices used to lower the maximum degree on the filtered graph $G^{(k)}_f$.

\section{Additional Practical Notes}

Note that care must be taken to not resize under a lock to maintain the expected contention time guarantee, which can be achieved by using incremental resizing, i.e., when a hash table is full, create a new one of double the size, and keep the old, moving an edge over from old to new on every insert rather than all at once during the resize. However, we found that this optimization was not critical for efficient performance in practice.

We also note that, while HT did not perform well relative to the heuristic FT in Sec.~\ref{sec:dimred}, it nonetheless provides a valuable technique for reducing memory usage in settings where a single index over features is unavailable. Frequently, sparse data is encoded in string/value pairs, without a global index. HT points us to a low-memory way of computing such an index. We construct a mapping using an ordinary hash table from 64-bit hashes of feature strings as keys themselves to 32-bit index. This can be done in a preliminary pass, in a parallel fashion, similar to Sec.~\ref{sec:graph-constr}, and yields a compact mapping. With such a large hash, collisions are exceedingly rare, with the probability of a collision among a billion features being less than 3\%, by a birthday paradox calculation.

Experiment code, configuration, and scripts to download the datasets are available in the GitHub repository \href{https://github.com/sisudata/chromatic-learning}{sisudata/chromatic-learning}.

\end{document}

%% file: defs.tex
\usepackage{bbm}
\usepackage{enumerate}
\usepackage{amsmath}
\usepackage{amsthm}
\usepackage{amssymb}
\usepackage{amsfonts}
\usepackage{mathrsfs}
\usepackage{mathtools}
\usepackage[all]{xy}

\usepackage{graphicx}
\usepackage{caption}

\usepackage{nicefrac}

\newcommand{\pa}[1]{ \left({#1}\right) }
\newcommand{\ha}[1]{ \left[{#1}\right] }
\newcommand{\ca}[1]{ \left\{{#1}\right\} }
\newcommand{\inner}[1]{\left\langle #1 \right\rangle}

\newcommand{\norm}[1]{\left\lVert #1 \right\rVert}
\newcommand{\card}[1]{\left\lvert{#1}\right\rvert}
\newcommand{\abs}[1]{\card{#1}}

\newcommand{\vv}{\mathbf{v}}

\newcommand{\vx}{\mathbf{x}}
\newcommand{\vy}{\mathbf{y}}

\newcommand{\R}{\mathbb{R}}

\newcommand{\N}{\mathbb{N}}

\newcommand{\mcF}{\mathcal{F}}
\newcommand{\mcG}{\mathcal{G}}

\DeclarePairedDelimiterX{\infdivx}[2]{(}{)}{ #1\;\delimsize\|\;#2 }

\makeatletter
\newcommand{\distas}[1]{\mathbin{\overset{#1}{\kern\z@\sim}}}%
\makeatother
\newcommand{\disteq}{\overset{d}{=}}

\DeclareMathOperator\mathExp{\mathbb{E}}

\newcommand{\E}{\mathExp}

\newcommand{\set}[2]{ \left\{ #1 \,\middle|\, #2 \right\} }
\newcommand{\CE}[2]{ \mathExp\left[ #1 \,\middle|\, #2 \right] }

\DeclareMathOperator*{\argmax}{argmax}

\makeatletter
\DeclareFontFamily{U}  {MnSymbolF}{}
\DeclareSymbolFont{symbolsMN}{U}{MnSymbolF}{m}{n}
\SetSymbolFont{symbolsMN}{bold}{U}{MnSymbolF}{b}{n}
\DeclareFontShape{U}{MnSymbolF}{m}{n}{
    <-6>  MnSymbolF5
   <6-7>  MnSymbolF6
   <7-8>  MnSymbolF7
   <8-9>  MnSymbolF8
   <9-10> MnSymbolF9
  <10-12> MnSymbolF10
  <12->   MnSymbolF12}{}
\DeclareFontShape{U}{MnSymbolF}{b}{n}{
    <-6>  MnSymbolF-Bold5
   <6-7>  MnSymbolF-Bold6
   <7-8>  MnSymbolF-Bold7
   <8-9>  MnSymbolF-Bold8
   <9-10> MnSymbolF-Bold9
  <10-12> MnSymbolF-Bold10
  <12->   MnSymbolF-Bold12}{}
\DeclareMathSymbol{\tbigtimes}{\mathop}{symbolsMN}{2}
\newcommand*{\bigtimes}{%
  \DOTSB
  \tbigtimes
  \slimits@ 
}
\makeatother

\newcommand\restr[2]{{
  \left.\kern-\nulldelimiterspace 
  #1 
  \vphantom{\big|} 
  \right|_{#2} 
  }}